\documentclass{article} 
\usepackage{PRIMEarxiv}

\usepackage{amssymb}            
\usepackage{mathtools}          
\usepackage{mathrsfs}           
\usepackage{graphicx}           
\usepackage{subcaption}         
\usepackage[space]{grffile}     
\usepackage{url}                
\usepackage{lipsum}             
\usepackage[acronym]{glossaries}
\usepackage{xcolor}
\usepackage{amsmath, amsthm, amssymb, dsfont}
\usepackage{float}
\usepackage{natbib}

\newacronym{acr:rl}{RL}{Reinforcement Learning}
\newacronym{acr:dqn}{DQN}{Deep Q-Network}
\newacronym{acr:ddqn}{DDQN}{Double Deep Q-Network}
\newacronym{acr:tql}{TQL}{Tabular Q-learning}
\newacronym{acr:sme}{SME}{Synthetic Monitoring Environment}
\newacronym{acr:ood}{OOD}{Out-of-Distribution}
\newacronym{acr:wd}{WD}{Within-Distribution}
\newacronym{acr:dun}{DUN}{Deep Uniform Network}
\newacronym{acr:pit}{PIT}{Probability Integral Transform}
\newacronym{acr:clt}{CLT}{Central Limit Theorem}
\newacronym{acr:cdf}{CDF}{Cumulative Distribution Function}
\newacronym{acr:ppo}{PPO}{Proximal Policy Optimization}
\newacronym{acr:sac}{SAC}{Soft Actor-Critic}
\newacronym{acr:td3}{TD3}{Twin Delayed Deep Deterministic Policy Gradient}
\newacronym{acr:mae}{MAE}{Mean Absolute Error}
\newacronym{acr:bc}{BC}{Behavioral Cloning}
\newacronym{acr:iql}{IQL}{Implicit Q-Learning}

\newtheorem{theorem}{Theorem}
\newtheorem{proposition}{Proposition}
  
\title{Synthetic Monitoring Environments for Reinforcement Learning
}

\author{%
  Leonard S.~Pleiss\\
  Technical University Munich\\
  Munich, 80331 \\
  \texttt{leonard.pleiss@tum.de} \\
  \And
  Carolin Schmidt \\
  Technical University Munich\\
  Munich, 80331 \\
  \texttt{carolin.schmidt@tum.de} \\
  \And
  Maximilian Schiffer \\
  Technical University Munich\\
  Munich, 80331 \\
  \texttt{schiffer@tum.de} \\
}

\begin{document}

\maketitle

\begin{abstract}
Reinforcement Learning (RL) lacks benchmarks that enable precise, white-box diagnostics of agent behavior. Current environments often entangle complexity factors and lack ground-truth optimality metrics, making it difficult to isolate why algorithms fail. We introduce Synthetic Monitoring Environments (SMEs), an infinite suite of continuous control tasks. SMEs provide fully configurable task characteristics and known optimal policies. As such, SMEs allow for the exact calculation of instantaneous regret. Their rigorous geometric state space bounds allow for systematic within-distribution (WD) and out-of-distribution (OOD) evaluation. We demonstrate the framework’s benefit through multidimensional ablations of PPO, TD3, and SAC, revealing how specific environmental properties--such as action or state space size, reward sparsity and complexity of the optimal policy--impact WD and OOD performance. We thereby show that SMEs offer a standardized, transparent testbed for transitioning RL evaluation from empirical benchmarking toward rigorous scientific analysis.
\end{abstract}

\section{Introduction}

\gls{acr:rl} has driven remarkable breakthroughs over the last decade, demonstrating super-human performance in domains ranging from complex strategy games and robotic manipulation to the fine-tuning of Large Language Models \citep{mnih_playing_2013, mnih_human-level_2015, silver2017masteringchessshogiselfplay, schrittwieser_mastering_2019, ouyang2022traininglanguagemodelsfollow}. These successes have been fueled by the availability of standardized benchmarks, e.g., the Arcade Learning Environment \citep{ArcadeLearningEnvironment} and MuJoCo \citep{Mujoco}, which have served as the proving grounds for novel algorithms. However, while these environments are effective for comparing relative performance, they remain opaque regarding the fundamental mechanics of the learning process. Further, as many empirical results are based on a small set of environments, it restricts the scope of empirical evaluations to a narrow set of environment characteristics and dynamics, possibly compromising the generalizability of demonstrated results \citep{whiteson_icml09}. 

As the field matures from demonstrating capability to requiring a scientific understanding of reliability and efficiency, the rigidity of current benchmarks has become a fundamental obstacle. Many existing environments do not provide sufficient infrastructure for the precise monitoring of learning dynamics, complicating the diagnosis of failure modes. Specifically, we identify three critical shortcomings in contemporary \gls{acr:rl} benchmarks:

\begin{enumerate}
\item \emph{Absence of ground-truth optimality measures:} In most popular benchmarks, the true optimal policy $\pi^{\star}$ is mathematically intractable. Consequently, a policy's \emph{degree of suboptimality} is unknown. This forces researchers to rely on relative performance metrics against human baselines or other algorithms, rather than measuring the absolute regret. This opacity makes it difficult to precisely monitor performance. It obscures whether an agent has achieved a globally optimal solution or is merely stalled at a local optimum of unknown quality.

\item \emph{Inability to quantify robustness and generalization:} While there is a growing interest in \gls{acr:ood} generalization \citep{nasvytis2024rethinkingoutofdistributiondetectionreinforcement}, current benchmarks frequently lack mechanisms to provide a systematic testbed. Robustness tests often rely on binary or qualitative measures, missing an exact continuous metric that defines the proximity of an \gls{acr:ood} state to the training distribution. Without precise, quantifiable measures of regret and distributional distance, it is difficult to rigorously assess an agent's robustness to varying degrees of distance to the training distribution, i.e., how well the agent performs as the scenarios it encounters become progressively more unfamiliar compared to the exact data it saw during its learning phase.

\item \emph{Entangled complexity and lack of configurability:} Key characteristics of an environment--such as the cardinality of the state and action spaces, reward sparsity, and environment complexity--are typically fixed or entangled. Increasing the difficulty of a standard benchmark, if at all supported, often inadvertently alters multiple complexity axes simultaneously. This lack of independent configurability prevents orthogonal ablation studies, making it difficult to isolate exactly which environmental property is causing an algorithm to fail.
\end{enumerate}

In summary, the opacity and rigidity of current benchmarks hinder the precise monitoring of algorithmic behavior. To bridge this gap, we propose \emph{\glspl{acr:sme}}. \glspl{acr:sme} are the result of a framework capable of generating unlimited and diverse continuous control tasks. By offering full control over environmental parameters, access to the optimal policy and a well-defined state space, \glspl{acr:sme} allow for the precise monitoring of learning dynamics under diverse task configurations, as well as rigorous \gls{acr:wd} and \gls{acr:ood} evaluation.

\paragraph{Related work}

 Several prior studies have identified the shortcomings of prevalent benchmarks and introduced alternative evaluation suites. Most comparable to our approach are propositions by \citet{osband2020behavioursuitereinforcementlearning}, \citet{derail} and \citet{chevalierboisvert2023minigridminiworldmodular}.

\cite{osband2020behavioursuitereinforcementlearning} proposed \textsc{bsuite}, a large set of highly specific parametrized experiments, aimed to isolate core dimensions relevant in \gls{acr:rl}, such as exploration, noise or scale. Yet, environments are limited to discrete actions, and while the set of experiments is vast, it does not allow for a free modulation of key environmental characteristics and does not allow for systematic and comprehensive within-distribution and out-of-distribution evaluation. Similarly, \citet{derail} proposed \textsc{Derail}, another benchmark containing various experiments aimed at isolating important environmental features to allow for the identification of failure modes. Importantly, different environmental features are compared within different environments. Further, the optimal policy remains opaque. \citet{chevalierboisvert2023minigridminiworldmodular} proposed \textsc{MiniGrid} and \textsc{MiniWorld}, minimalistic 2D and 3D grid environments, which enable the modulation of some key environmental characteristics such as environment size, reward density, and exploration difficulty. Under limited problem complexity, it even allows for a calculation of optimal policies. However, the nature of the environment limits the ability to modulate the state and action space dimensionality, and also restricts the setting to discrete state and action spaces. Other benchmarks aim to isolate specific task characteristics. \textsc{ProcGen} \citep{cobbe2020leveragingproceduralgenerationbenchmark} leverages procedural generation of game-like settings to assess \gls{acr:ood} generalization. \textsc{RobustGymnasium} \citep{gu2025robustgymnasiumunifiedmodular} systematically applies diverse environmental disruptions to rigorously assess agent robustness.

\paragraph{Contribution}

Despite prior efforts, existing benchmarks lack the mechanisms required for a complete white-box analysis of the learning process. Such an analysis fundamentally relies on the ability to systematically isolate task characteristics and evaluate both \gls{acr:wd} and \gls{acr:ood} performance with absolute precision. Such rigorous evaluations fundamentally depend on three prerequisites: the free configurability of key environment characteristics, the precise quantification of distance to the training manifold for \gls{acr:ood} states, and the exact calculation of instantaneous regret for any state-action pair under arbitrary task complexity. \glspl{acr:sme} address these gaps within a unified framework, expanding the scope of available monitoring capabilities..

Specifically, our contribution is fourfold: First, we introduce \glspl{acr:sme}--modular, highly customizable environments designed for the precise evaluation of reinforcement learning agents. Second, we theoretically substantiate this framework by analyzing its core mechanisms, namely two classes of measure-preserving functions that serve as generalizable transition kernels and a priori optimal policies. Third, we leverage \glspl{acr:sme} to conduct rigorous ablation studies across isolated environmental characteristics, overcoming the confounding factors present in standard benchmarks. Finally, we establish a standardized methodology for both within-distribution (\gls{acr:wd}) and systematic \gls{acr:ood} evaluation, and present empirical insights into agent robustness.

\section{Synthetic monitoring environments}

Having established the limitations of current benchmarks in isolating specific learning challenges, we introduce \textit{Synthetic monitoring environments} (SMEs). This class of environments is designed to bridge the gap between toy problems like \textsc{GridWorld} and complex, high-dimensional tasks like \textsc{MuJoCo} or \textsc{Atari}, retaining the analytical tractability of the former while offering the flexibility to scale toward the complexity of the latter.

 SMEs operate over continuous state and action spaces on the unit hypercube, employing carefully designed neural networks to parameterize transition dynamics and optimal policies. This design offers five distinct advantages over standard \gls{acr:rl} environments:

\begin{enumerate}
    \item \emph{Complete configurability:} Unlike standard benchmarks where environment characteristics are fixed (e.g., the 8-dimensional action space in \textsc{Ant}), \glspl{acr:sme} allow for the independent modulation of state dimensionality ($N_s$), action dimensionality ($N_a$), reward distribution frequency ($k$), reward sparsity ($r_{min}$), survival difficulty $\mathcal{D}$ and optimal policy complexity $\mathcal{C}_{\pi^{\star}}$. This enables an isolated evaluation along single task complexity axes.
    
    \item \emph{Ground-truth optimality:} In most complex tasks, the true optimal policy $\pi^{\star}$ is unknown, and performance is measured by asymptotic return, which is often a noisy proxy. In \glspl{acr:sme}, the optimal policy is generated \textit{a priori} alongside the environment. This allows us to calculate the exact \textit{instantaneous regret} at every time step, even under extremely complex environment dynamics.
    
    \item \emph{Comprehensive within-distribution evaluation:} During training, the state space is rigorously bounded within the unit hypercube $\mathcal{S} \in [0, 1]^{N_s}$. In conjunction with access to the optimal policy, this allows for comprehensive and precise evaluation across the entirety of viable states observable during training.
    
    \item \emph{Native out-of-distribution support:} The clear geometric bounds of \glspl{acr:sme} during training facilitate precise \gls{acr:ood} testing. As the optimal policy provides actions for any conceivable state, we may evaluate the agent on states beyond the borders of the unit hypercube to test \gls{acr:ood} performance. This provides a standardized and precise testbed for measuring generalization, a capability currently missing in popular benchmarks where generating \gls{acr:ood} states is non-trivial.
    
    \item \emph{Functional diversity:} \glspl{acr:sme} constitute a highly diverse, dynamic suite of environments, representing an infinite variety of vastly different and arbitrarily complex tasks.
\end{enumerate}

Collectively, these characteristics enable a \emph{white-box analysis of black-box agents} across fully modular tasks. This flexibility allows us to rigorously stress-test agents, identifying not just \textit{that} they fail, but specifically \textit{where} and \textit{why} they fail.

\section{Methodology}

\glspl{acr:sme} are fundamentally characterized by two key elements, the transition kernel $\mathcal{T}$, and an optimal policy $\pi^{\star}$. The transition kernel defines the flow of states based on actions taken by the agent. The optimal policy poses a learning target. This allows us to calculate rewards as a function of the deviation between the agent's actions and the optimal actions provided by the optimal policy.

\subsection{Transition kernel} To govern the environment dynamics, the transition function of an \gls{acr:sme} must map a state-action pair of arbitrary dimensionality, $(s_t, a_t) \in [0, 1]^{N_s} \times [0, 1]^{N_a}$, to a subsequent state $s_{t+1} \in [0, 1]^{N_s}$, while satisfying three key criteria: (i) ensuring non-trivial contributions from all state and action channels; (ii) preserving the measure of the state distribution to prevent state-space collapse, which would otherwise degenerate the learning task; and (iii) maintaining determinism within a specific environment while providing diversity across instantiations.

To simultaneously satisfy these requirements, we formulate the transition dynamics as a composition of an affine transformation and a bounded non-linear activation. Specifically, the next state is defined as $s_{t+1} = \psi \left( s_t + a_t W + b \right)$, where $W \in \mathbb{R}^{N_a \times N_s}$ is a fixed weight matrix, $b \in \mathbb{R}^{N_s}$ is a fixed bias vector, and $\psi$ is a topological activation function.

\paragraph{Affine transformation} Unlike standard Gaussian initialization used in many neural architectures, we initialize the weight matrix $W$ to be row-stochastic. Specifically, the weights are drawn from a uniform distribution $\mathcal{U}(0, 1)$ and normalized such that the sum of weights associated with each action dimension equals unity, $\sum_{j=1}^{N_s} W_{ij} = 1, \quad \forall i \in \{1, \dots, N_a\}$. This initialization ensures that the magnitude of each action component $a_{t,i}$ is conserved as it is distributed across the state dimensions, and prevents the vanishing or exploding signal variances often found in deep neural networks. The bias $b$ is initialized uniformly, $b \sim \mathcal{U}(0, 1)^{N_s}$, to introduce a stochastic offset to the dynamics.

\paragraph{Triangle wave activations} To constrain the state space to the unit hypercube while preserving continuity, we employ a normalized triangle wave activation function $\psi: \mathbb{R} \to [0, 1]$, applied element-wise, $\psi(x) = \frac{1}{\pi} \arccos\left(\cos(2\pi x)\right)$.

Beyond simple boundary enforcement, this specific choice of activation function provides a rigorous statistical guarantee for the environment dynamics. Because the triangle wave acts as a continuous folding mechanism--analogous to the chaotic Tent Map over a periodic domain--, it avoids state-space collapse and achieves exact measure preservation. In the context of \gls{acr:rl}, this exact measure preservation is critical. It mathematically guarantees that the transition dynamics do not inherently compress the state space into narrow manifolds or point attractors, preserving the integrity of the learning task. We refer to the Supplementary Material, Section~\ref{sec:transition_discussion} for a formal analysis.

\subsection{Optimal policy}

The optimal policy of a \gls{acr:sme} must map a current state of arbitrary dimensionality $s_t \in \mathcal{S} = [0, 1]^{N_s}$ to an action of arbitrary dimensionality $a_t \in \mathbb{R}^{N_a}$. To serve as a rigorous learning target for the \gls{acr:rl} agent, it must satisfy three key criteria: (i) exhibit arbitrarily tunable complexity while remaining statistically well-behaved; (ii) maintain full coverage of the action space across the entire state space, thereby preventing uncontrolled collapse of output complexity; and (iii) guarantee determinism within a specific environment while ensuring diversity across instantiations.

To satisfy these requirements, we construct the optimal policy $\pi^{\star}: \mathcal{S} \to \mathcal{A}$ using a specialized architecture we denote as a \gls{acr:dun}. 
The network is composed of a sequence of Uniform Layers, each designed to map a uniform input distribution $u \sim \mathcal{U}(0,1)^{N_{in}}$ to a uniform output $v \sim \mathcal{U}(0,1)^{N_{out}}$, where $N_{in}$ and $N_{out}$ denote the number of input and output neurons.

\paragraph{The uniform layer} Each layer $l$ performs a transformation $f_l: [0, 1]^{N_{in}} \to [0, 1]^{N_{out}}$ consisting of an affine projection followed by a specific non-linearity. Given an input vector $x$ where each element $x_i \sim \mathcal{U}(0, 1)$, we first compute a pre-activation vector $z$, $z = Wx + b$.

 We employ a custom variance-correcting weight initialization, as standard methods \citep{glorot, he2015delvingdeeprectifierssurpassing} fail to account for the specific variance of our $\mathcal{U}(0,1)$ distribution. To maximize feature independence, we initialize $W$ as a semi-orthogonal matrix (strictly orthogonal if $N_{in}=N_{out}$). Given that the uniformly distributed inputs possess a variance of $\text{Var}(x_i) = 1/12$, we scale the weight matrix $W$ by a factor of $\sqrt{12}$. This enforces unit variance across all pre-activations, $\text{Var}(z_j) = 1$. Furthermore, given $\mathbb{E}[x_i] = 0.5$, we deterministically center the pre-activations at zero by setting the bias to $b = -0.5 W \mathds{1}$. By the \gls{acr:clt}, the linear combination $z$ converges to a standard normal distribution ($z \sim \mathcal{N}(0, I)$) as $N_{in}$ grows. Crucially, as demonstrated in the Supplementary Material, Section~\ref{sec:policy_discussion}, even for small values of $N_{in}$, this distributional stability is largely preserved, preventing policy collapse. Finally, to project $z$ back onto the unit hypercube, we apply the standard normal \gls{acr:cdf}, $\Phi(\cdot)$, as our activation function: $y = \Phi(z) = \frac{1}{2} \left[ 1 + \text{erf}\left( \frac{z}{\sqrt{2}} \right) \right]$. According to the \gls{acr:pit}, transforming a random variable $Z$ by its continuous \gls{acr:cdf} $F_Z$ yields $Y = F_Z(Z) \sim \mathcal{U}(0,1)$, thereby guaranteeing that the output $y$ is marginally distributed as $\mathcal{U}(0,1)^{N_{out}}$.

\paragraph{The full network} The full policy $\pi^{\star}$ consists of a stack of uniform layers. The depth of the network allows for complexity control. For the low-complexity regime ($L=1$), the mapping yields a near-linear transformation that largely preserves the input topology. Conversely, in the deep regime ($L \gg 1$), the composition of successive projection-activation layers induces a highly non-linear transformation. This leads to a marked decay in structural retention as the number of layers increases, as displayed in Figure~\ref{fig:policy-complexity}.

\begin{figure}[htbp]
    \centering
    \vspace{-0.3cm}
    \includegraphics[width=1\linewidth]{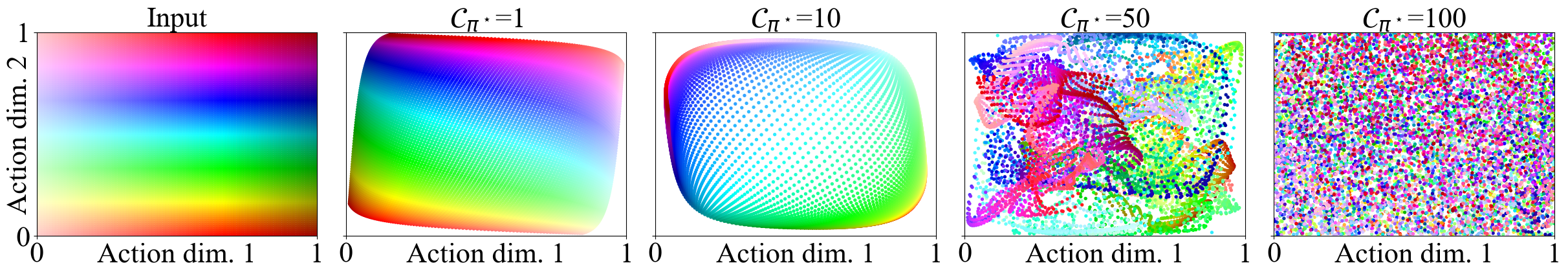}
    \caption{Topological deformation of the input space as a function of the number of stacked uniform layers, $\mathcal{C}_{\pi^{\star}}$.}
    \label{fig:policy-complexity}
\end{figure}

This optimal policy function has two key advantages: First, it maintains distributional stability. In contrast to standard architectures that frequently exhibit policy saturation at the action space boundaries, the \gls{acr:dun} facilitates a more uniform utilization of the action manifold, ensuring the optimal policy leverages the entire available control range and preventing a corruption of learning dynamics. Second, the explicit orthogonalization step minimizes redundant feature learning, ensuring that the target policy is information-dense. For a thorough formal analysis as well as extended empirical evaluations, we refer to the Supplementary Material, Section~\ref{sec:policy_discussion}.

\subsection{Reward Formulation and Episode Dynamics}

The reward mechanism in \glspl{acr:sme} is strictly based on a function of the behavioral deviation between the agent policy $\pi$ and the optimal policy $\pi^\star$, evaluated at each time step. To decouple the instantaneous performance evaluation from the frequency of the reward signal provided to the agent, we separate the internal step reward calculation from the environmental reward distribution.

\paragraph{Step reward calculation} At any given time step $t$, the agent emits an action $a_t \in \mathcal{A}$. The optimal action is concurrently determined by the optimal policy, $a^\star_t = \pi^\star(s_t)$. We first compute the baseline similarity $\tilde{r}_t$ via the complement of the \gls{acr:mae} between the two action vectors, $\tilde{r}_t = 1 - \frac{1}{N_a} \|a_t - a^\star_t\|_1$. This baseline similarity is rescaled and bounded. Because an average $\tilde{r}_t$ of $0.75$ can be trivially achieved by static, state-independent policies (such as predicting the center of the action space, $a_t = 0.5 \mathbf{1}_{N_a}$), we apply a linear transformation that clamps any value below this threshold to zero, $\hat{r}_t = \max\left(0, 4 (\tilde{r}_t - 0.75)\right)$. This ensures the agent is only rewarded for meaningful alignment with the optimal policy, de-noising the reward signal. Finally, to introduce controllable reward sparsity, we gate the instantaneous step reward $r_t$ by a minimum reward threshold $r_{\min}$, $r_t = \hat{r}_t \cdot \mathds{1}_{\{\hat{r}_t > r_{\min}\}}$.

\paragraph{Reward distribution and state augmentation} While we calculate $r_t$ internally at every step, it is not necessarily provided to the agent immediately. To facilitate the independent modulation of reward distribution frequency, we accumulate the step rewards into a running sum $r_{\text{cum}}$. The environment only dispenses a non-zero reward $R_t$ to the agent at intervals defined by a frequency parameter $k$, or upon episode truncation, $R_t = r_{\text{cum}} \cdot \mathds{1}_{\{(t \mod{k} = 0) \lor \text{truncated}\}}$. Following a non-zero payout, $r_{\text{cum}}$ is reset to zero. To preserve the Markov property under this delayed reward scheme, the state observed by the agent is augmented with auxiliary information. In addition to the raw topological state $s_t$, the observation includes the normalized step count $t/T$ where $T$ is the episode length, and the normalized accumulated reward $r_{\text{cum}} / k$.

\paragraph{Termination and truncation} We impose two distinct end-of-episode conditions. \emph{Termination} occurs dynamically if the agent's performance degrades beyond an acceptable bound. If the step reward $r_t$ falls below the survival difficulty threshold $\mathcal{D}$, the episode ends immediately. \emph{Truncation} occurs when the agent reaches the predefined temporal horizon of the environment, $t = T$.

\subsection{Integration} 

The agent starts in a state randomly sampled from the unit hypercube, $S_{init} \sim \mathbb{U}([0,1]^{N_s})$. At each discrete time step $t$, the agent observes the current state vector $s_t \in \mathcal{S} = [0, 1]^{N_s}$ and executes an action $a_t \in \mathcal{A} = [0, 1]^{N_a}$. The environment simultaneously computes the oracle action $a^{\star}_t = \pi^{\star}(s_t)$, and computes the reward $R_t = f(a_t, a^{\star}_t)$. It then updates the internal state via the transition function $\mathcal{T}(s_t, a_t)$. We repeat this procedure until the trajectory is terminated due to insufficient agent performance, or truncated due to a given maximum episode length. Consequently, the maximum achievable return is $T$, whereas the best trivial policy--always predicting $0.5 \mathds{1}_{N_a}$--would yield a return of $0$. Naturally, the latter only applies when $r_{min} = \mathcal{D} = 0$.

\subsection{Systematic within- and out-of-distribution evaluation}

To evaluate a policy under \gls{acr:wd} and \gls{acr:ood} regimes, we sample states within and beyond the unit hypercube, which bounds the training state space. We then evaluate the performance of our agents against the optimal policy based on $\tilde{r}_t$.\footnote{Note that we employ the \gls{acr:mae} instead if the reward as a performance criterion, as the reward is deliberately designed to provide no signal for large deviations to prevent a noisy reward signal. It therefore does not differentiate bad actions, which is undesirable for a final evaluation.} To assess \gls{acr:ood} performance, we define a sequence of $N$ nested hypercubes $\mathcal{X}_{\epsilon}$ centered at $c = [0.5, \dots, 0.5]^\top \in \mathbb{R}^{N_s}$. For an expansion factor $\epsilon \geq 0$, we define the expanded state space via the $\ell_\infty$-norm, $\mathcal{X}_{\epsilon} = \left\{ s \in \mathbb{R}^{N_s} \mid \|s - c\|_\infty \leq \frac{1+\epsilon}{2} \right\}$, where $\mathcal{X}_0 = [0, 1]^{N_s}$ represents the nominal unit hypercube. We partition the \gls{acr:ood} space into a set of mutually exclusive categories $\{\mathcal{C}_m\}_{m=1}^N$. A state $s$ belongs to the $m$-th \gls{acr:ood} category, corresponding to the interval $(\epsilon_{m-1}, \epsilon_m]$, if and only if $s \in \mathcal{X}_{\epsilon_m} \setminus \mathcal{X}_{\epsilon_{m-1}}$. This formulation ensures that each state is mapped to a unique category based on its maximum axial deviation from the nominal boundaries. Specifically, a state $s$ is categorized by its expansion level $E(s) = 2\|s - c\|_\infty - 1$; it belongs to $\mathcal{C}_m$ if $E(s) \in (\epsilon_{m-1}, \epsilon_m]$.

\section{Numeric studies}

 To demonstrate the empirical utility of our framework, we evaluate three canonical \gls{acr:rl} algorithms--\gls{acr:sac} \citep{haarnoja_soft_2018}, \gls{acr:td3} \citep{fujimoto_td3}, and \gls{acr:ppo} \citep{schulman2017proximalpolicyoptimizationalgorithms}--across a diverse set of \gls{acr:sme} configurations. We select these baselines due to their widespread adoption and to ensure robust algorithmic variety, allowing us to assess our testbed across both on-policy and off-policy methods, as well as stochastic and deterministic optimization paradigms. We perform ablations across six task complexity dimensions to isolate environment-induced variations in learning behavior, and systematically evaluate \gls{acr:wd} and \gls{acr:ood} performance. Importantly, the primary objective of the numeric studies conducted in this paper is explicitly not to yield novel empirical discoveries. Instead, they serve a dual purpose: to showcase examples of the diverse experimental configurations, analyses, and insights enabled by \glspl{acr:sme}, and to provide rigorous verification. By confirming that these configurations produce sound learning dynamics under which established algorithms behave as expected, we establish that the environment functions as intended. For hyperparameters, we used the standard specification from the \textsc{StableBaselines3} library \citep{antonin_raffin_stable-baselines3_2024}. We set $T = 100$. We conducted 400 evaluations, evenly spaced throughout the training procedure. Each evaluation consisted of five trajectories, with the agent's final score computed as the average return across these runs.

\paragraph{Ablations.} We define a default specification, comprising $N_{s} = 8$ state channels and $N_{a} = 4$ action channels, reward distribution interval $k=1$, minimal distributed step reward $r_{min} = 0$, optimal policy complexity $\mathcal{C}_{\pi^{\star}} = 1$, and survival difficulty $\mathcal{D}=0$. Based on this default, we compute ablations across different environmental characteristics. Figure~\ref{fig:ppo_sac_ablations} displays our findings. We provide an enlarged variant of Figure~\ref{fig:ppo_sac_ablations} in the Supplementary Material, Section~\ref{sec:ablations_fullpage}.

\begin{figure}[htbp]
    \centering
    \vspace{-0.3cm}
    \includegraphics[width=1\linewidth]{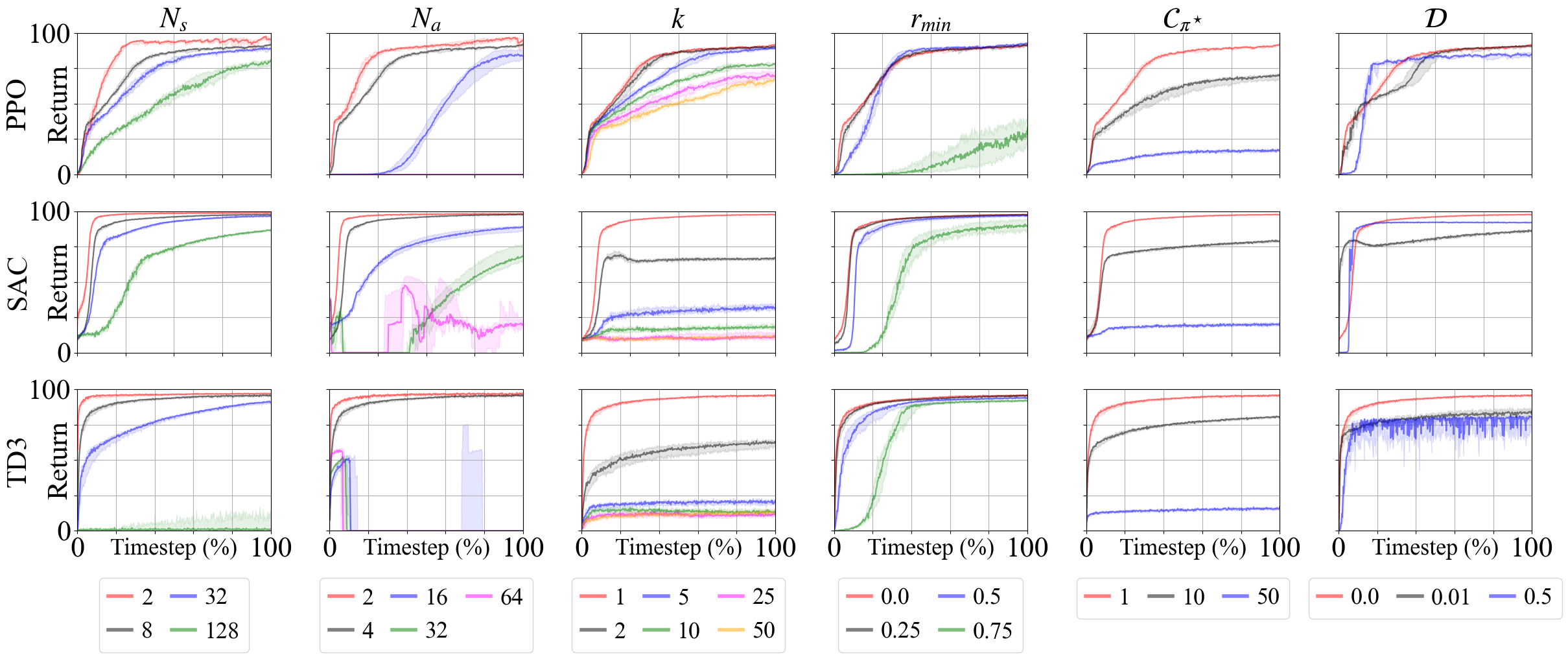}
    \caption{Ablations for PPO, SAC and TD3 across different task configurations over $10$ seeds. Curves indicate median evaluation performance, smoothed over $10$ points. Shaded areas indicate interquartile ranges.}
    \label{fig:ppo_sac_ablations}
\end{figure}

 Our findings reveal that the evaluated algorithms exhibit distinct sensitivities to the modulation of environmental characteristics. Specifically, \gls{acr:ppo} handles large reward distribution intervals better than \gls{acr:td3} and \gls{acr:sac}. This divergence is expected, as \gls{acr:ppo}'s use of Generalized Advantage Estimation allows it to robustly assign credit over extended horizons. Conversely, we find that \gls{acr:ppo} is more susceptible to larger minimum rewards, while \gls{acr:sac} displays the highest robustness to expansive state and action spaces. Finally, \gls{acr:td3} performs exceptionally well in simpler settings--likely due to the high sample efficiency of deterministic updates--but its performance deteriorates most rapidly as dimensionality increases.

\paragraph{Evaluation.} For each trial, we performed comprehensive \gls{acr:wd} and \gls{acr:ood} evaluations. We evaluate the agent performance across $50{,}000$ states within six disjoint regions, defined by the unit hypercube expansion set $\mathcal{E} = \{0.2m \mid m = 0, \dots, 5\}$. Results for an exemplary experiment are displayed in Figure~\ref{fig:evaluation_results}. This experiment employed the standard configuration across three levels of optimal policy complexity $\mathcal{C}_{\pi^{\star}}$, yet with $N_s = N_a = 2$ for intuitive visualization. Full evaluation results for all trials from the ablation study are displayed in Appendix~\ref{sec:full_evaluation_results}, Figure~\ref{fig:heatmap_fullpage}.

\begin{figure}[htbp]
    \centering
    \vspace{-0.3cm}
    \includegraphics[width=1\linewidth]{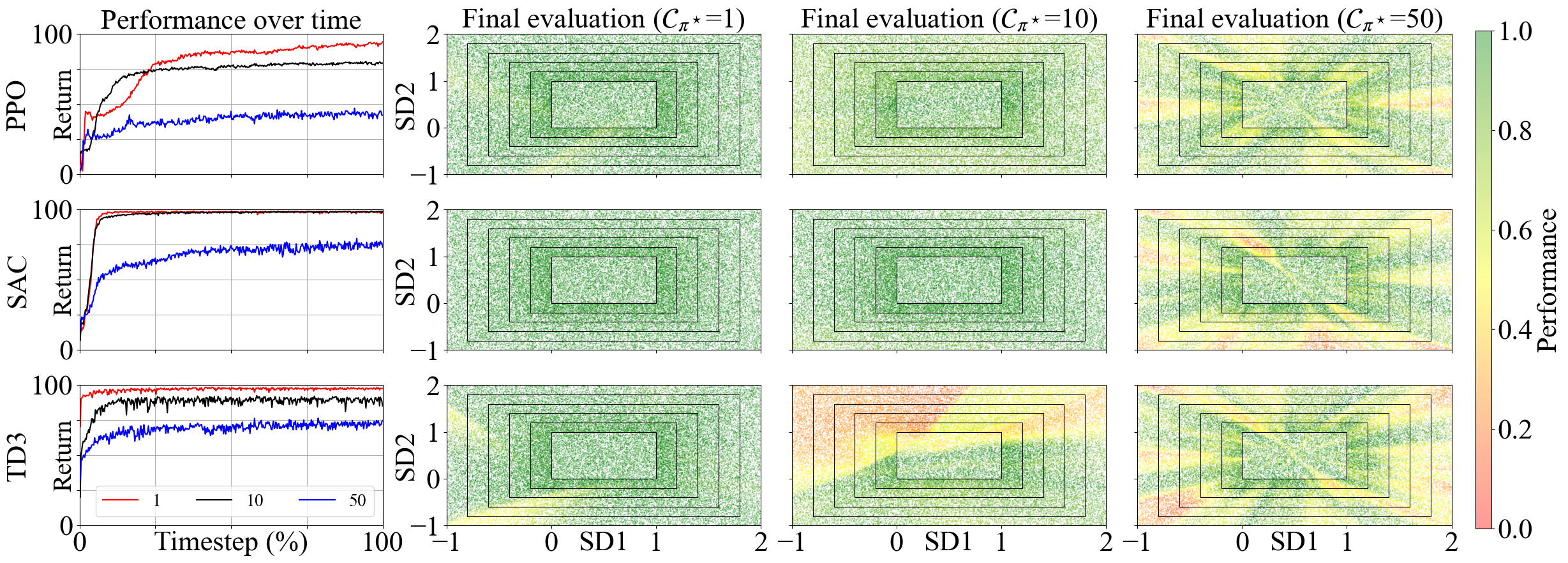}
    \caption{Evaluation performance during training (column 1) and final within-distribution and out-of-distribution performance for PPO, SAC and TD3 across different complexities of the optimal policy the agent seeks to mirror ($\mathcal{C}_{\pi^{\star}}$, columns 2-4). Performance is defined as the complement of the mean average error between action and optimal action $\tilde{r}_t$. States within $\mathbb{R}^2(0,1)$ are within-distribution. States beyond the unit square are out-of-distribution. SD = State Dimension.}
    \label{fig:evaluation_results}
\end{figure}


Our findings reveal that performance decays with distance to the training manifold. Across all experiments, we find an average decrease by $1.38\%$ $(\text{SD}=1.90)$ in the $\text{OOD}_{0-20\%}$ setting compared to the \gls{acr:wd} condition, which increases to a $5.10\%$ $(\text{SD} = 4.70)$ in the $\text{OOD}_{81-100\%}$ setting. We further observe a positive association between \gls{acr:wd} performance and the extent of performance decrease when moving \gls{acr:ood} across all experiments, $\cos(\text{WD}, \text{WD}-\text{OOD}_{81-100\%} ) = 0.66$.

\section{Discussion}

 \paragraph{Limitations} The core advantage of \glspl{acr:sme}--their exact analytical tractability--necessitates specific structural trade-offs. To guarantee precise monitoring and distributional stability, the framework relies on measure-preserving mappings. Consequently, \glspl{acr:sme} are designed to evaluate fundamental learning dynamics rather than to simulate extreme topological irregularities, such as discontinuous dynamics, severe bottlenecks, or heavy-tailed distributions encountered in some real-world domains. Additionally, because the measure preservation of the \gls{acr:dun} relies on the \gls{acr:clt}, exact preservation is marginally relaxed in low-dimensional spaces. However, as demonstrated in the Supplementary Material, Section~\ref{sec:policy_discussion}, the policy does not collapse, ensuring that the integrity of the learning task remains uncorrupted. Finally, because the optimal policy may naturally saturate in far-\gls{acr:ood} regions, \gls{acr:ood} metrics should be interpreted thoughtfully when evaluating agents that employ strict action clipping, particularly prior to convergence. 

 \paragraph{Future directions} The experiments showcased in this work represent only a fraction of the potential applications for \glspl{acr:sme}. A particularly promising direction is offline RL, where the ability to precisely control dataset quality offers significant potential for granular analytical insights. We provide an exemplary offline \gls{acr:rl} experiment demonstrating these diagnostic benefits in the Supplementary Material, Section~\ref{sec:offline_rl}. Beyond offline learning, the unique mathematical and structural properties of \glspl{acr:sme} open several exciting avenues for future research across domains such as continual and non-stationary learning, safe \gls{acr:rl}, and representation learning. 

\paragraph{Conclusion} In this work, we proposed \glspl{acr:sme}, a novel class of continuous control benchmarks designed to bridge the gap between analytically tractable toy problems and complex, high-dimensional tasks. We detailed the theoretical foundations of the environment's transition kernel and \gls{acr:dun}-driven optimal policies, and empirically demonstrated their distinct advantages. By providing complete parametric configurability, ground-truth access to the optimal policy, and mathematically precise state space boundaries enabling rigorous \gls{acr:wd} and \gls{acr:ood} evaluation, \glspl{acr:sme} offer a standardized testbed for a comprehensive analysis of learning dynamics and difficulty factors. We hope that this framework will be of significant use in future research, equipping the \gls{acr:rl} community with the diagnostic tools necessary to meticulously monitor, analyze, and ultimately improve \gls{acr:rl} algorithms.

\newpage

\bibliographystyle{unsrtnat}  
\bibliography{references}  

\newpage

\appendix

\section{Full evaluation results}\label{sec:full_evaluation_results}

\begin{figure}[h!]
    \centering
    \includegraphics[width=.7\linewidth]{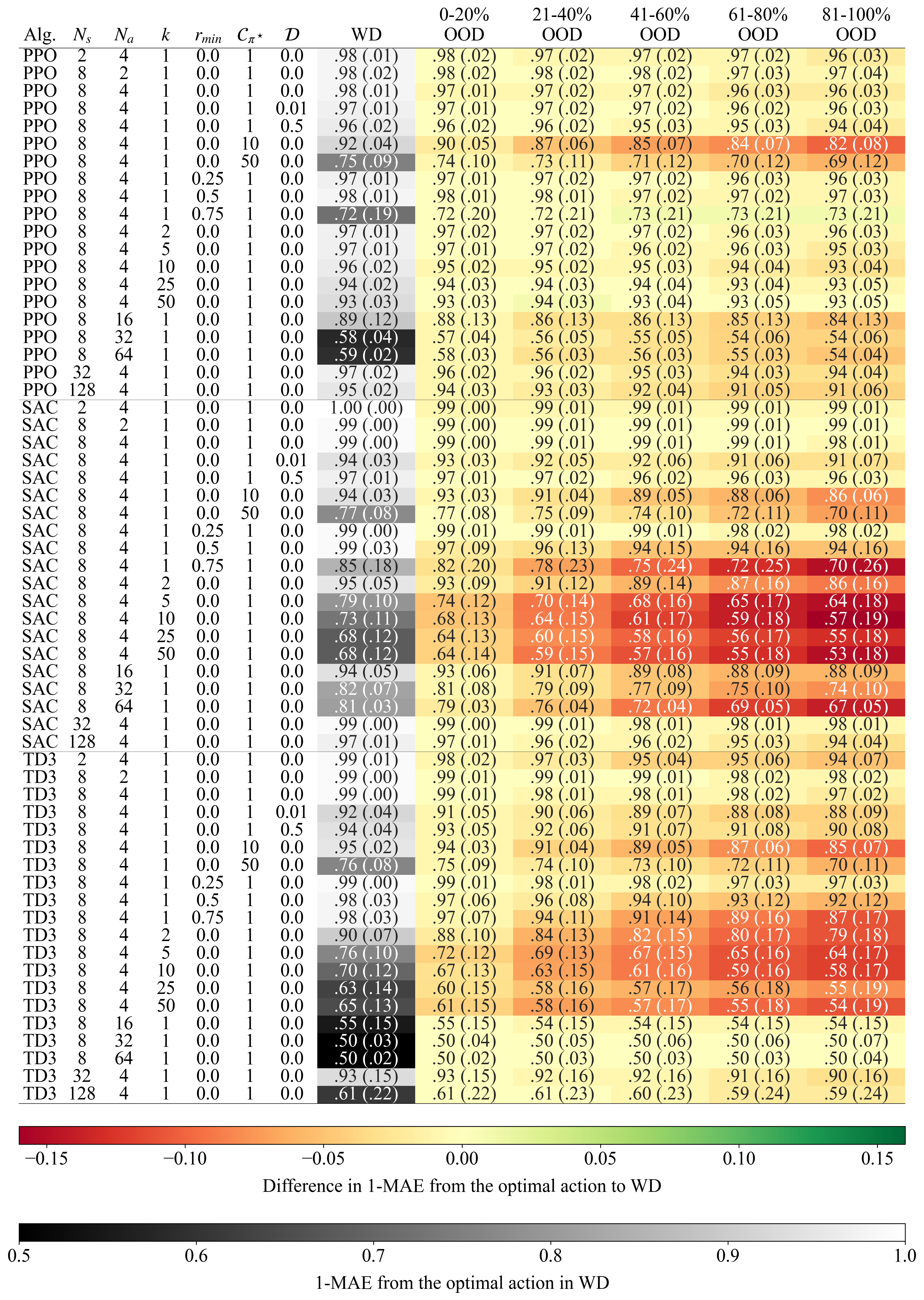}
    \caption{Performance by proximity to the training distribution for PPO, SAC and TD3 across 10 seeds. Alg. = Algorithm, WD = Within-distribution, OOD = Out-of-distribution, MAE = Mean Average Error.}
    \label{fig:heatmap_fullpage}
\end{figure}

\section{Discussion of the optimal policy}\label{sec:policy_discussion}

To maintain the integrity of the learning environment, the optimal policy must be approximately measure-preserving. Specifically, given uniformly distributed states, the resulting optimal actions must be sufficiently diverse. This property is crucial for ensuring consistent difficulty and comparability across environment instantiations. If the target policy were to collapse into a narrow action subspace, it would inadvertently trivialize the objective, allowing the agent to succeed using simple, state-independent policies. By maintaining sufficient action diversity, we prevent this corruption of the learning dynamics.

\subsection{Theoretical analysis of measure preservation}

For the optimal policy, the architecture relies on the \gls{acr:clt} over finite-width linear transformations. Passing the resulting Irwin-Hall distribution through the standard normal \gls{acr:cdf} asymptotically approximates a uniform distribution.

\begin{theorem}[Asymptotic preservation of the uniform measure]
\label{thm:policy_asymptotic}
Let $X \in \mathbb{R}^n$ be a random vector representing the input with independent components $X_i \sim \mathcal{U}(0,1)$. Let $T: \mathbb{R}^n \to \mathbb{R}^m$ be a layer of the optimal policy mapping defined as $T(X) = \Phi(W(X - \mathbf{\mu}))$, where $\mu_i = 0.5$, $\Phi$ is the standard normal CDF applied element-wise, and $W \in \mathbb{R}^{m \times n}$ is a weight matrix with mutually orthogonal rows satisfying $WW^\top = 12 I_m$. As $n \to \infty$, the output $Y = T(X)$ converges weakly to the joint uniform distribution $\mathcal{U}(0,1)^m$.
\end{theorem}

\begin{proof}
Define the centered input vector $\tilde{X} = X - 0.5$. Because $X_i \sim \mathcal{U}(0,1)$, the components of $\tilde{X}$ are independent, zero-mean random variables with variance $\text{Var}(\tilde{X}_i) = \frac{1}{12}$. Consequently, the covariance matrix of the centered input is $\Sigma_{\tilde{X}} = \frac{1}{12} I_n$.

Let $Z = W\tilde{X}$ be the pre-activation vector. By linearity of expectation, $\mathbb{E}[Z] = \mathbf{0}$. The covariance matrix of $Z$ is
\begin{equation}
\Sigma_Z = W \Sigma_{\tilde{X}} W^\top = W \left(\frac{1}{12} I_n\right) W^\top = \frac{1}{12} (W W^\top).
\end{equation}
By the premise of the theorem, $WW^\top = 12 I_m$, which yields $\Sigma_Z = I_m$. Thus, the components of $Z$ are uncorrelated and have unit variance.

Because $Z$ is obtained via a linear combination of independent, bounded random variables, we invoke the Multivariate \gls{acr:clt}. As $n \to \infty$, the distribution of $Z$ converges weakly in distribution to a standard multivariate Gaussian, $Z \xrightarrow{d} \mathcal{N}(\mathbf{0}, I_m)$. Crucially, the identity covariance implies that the components of $Z$ become mutually independent standard normal random variables in the limit.

Next, we apply the Probability Integral Transform (PIT). For any scalar random variable $Z_j \sim \mathcal{N}(0,1)$, applying its continuous CDF $\Phi$ yields $\Phi(Z_j) \sim \mathcal{U}(0,1)$. By the Continuous Mapping Theorem, since $Z \xrightarrow{d} \mathcal{N}(\mathbf{0}, I_m)$, the component-wise application of $\Phi$ implies:
\begin{equation}
Y = \Phi(Z) \xrightarrow{d} \mathcal{U}(0,1)^m
\end{equation}
By induction, stacking $K$ such layers preserves the weak convergence to the uniform distribution at each step, thereby asymptotically preserving the entropy of the input distribution.
\end{proof}

\subsection{Empirical analysis of measure preservation}

As the outlined approach for generating optimal actions relies on the \gls{acr:clt}, we may not theoretically guarantee that the optimal policy is measure-preserving, especially for low-dimensional states. However, as shown in Figure~\ref{fig:policy_by_complexity_and_input}, the optimal policy does not collapse, and approximate uniformity is recovered very quickly as complexity and/or the number of inputs increase. Consequently, the optimal policy preserves the integrity of the learning task, even for low-dimensional states.

\begin{figure}[h!]
    \centering
    \includegraphics[width=1\linewidth]{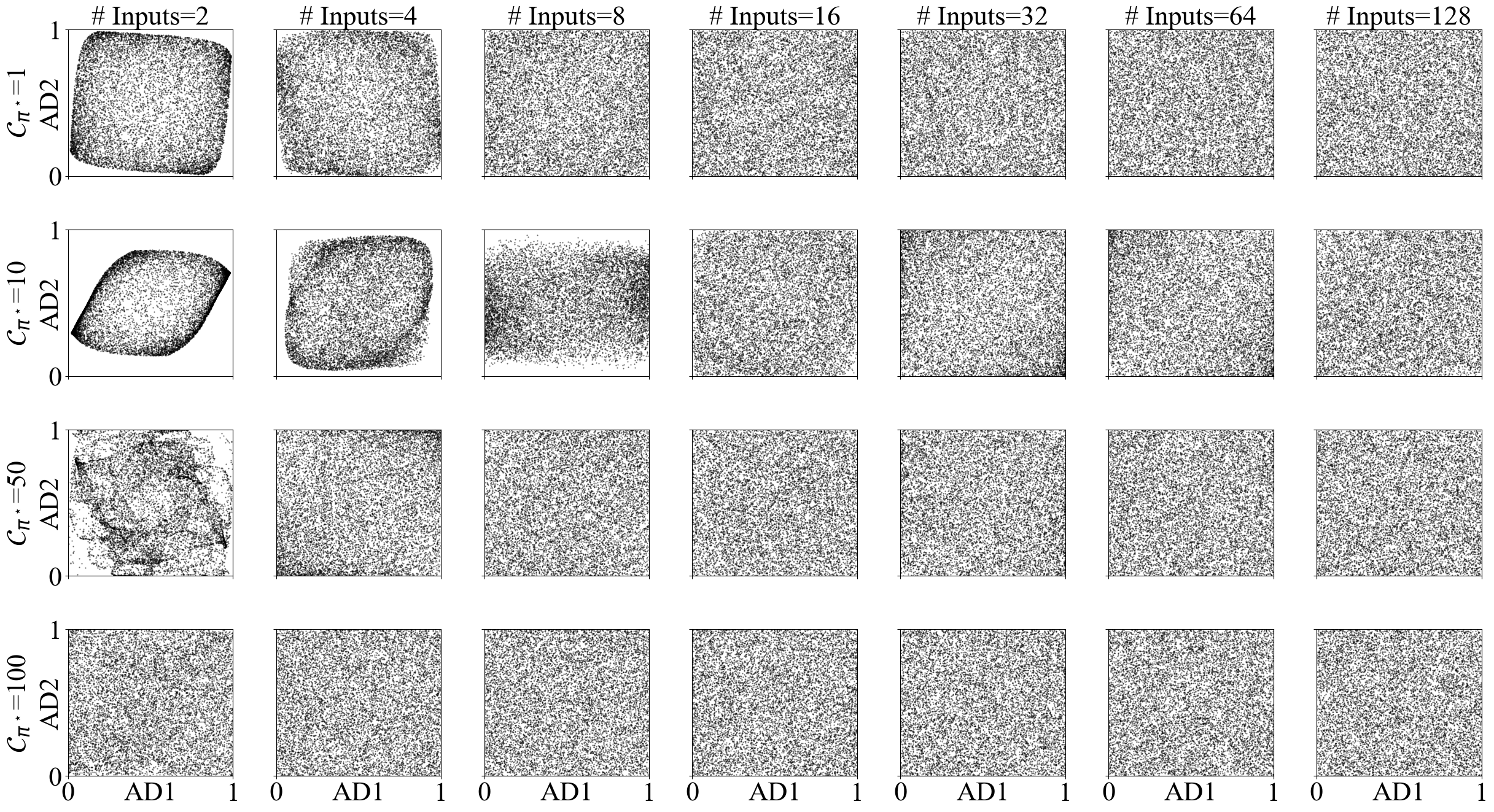}
    \caption{Optimal action distribution under varying input dimensions $N_s$ and complexities $\mathcal{C}_{\pi^{\star}}$ for $10{,}000$ states, uniformly sampled from the unit hypercube that bounds the state space. AD = Action dimension.}
    \label{fig:policy_by_complexity_and_input}
\end{figure}

\section{Discussion of the transition kernel}\label{sec:transition_discussion}

To yield robust and analytically sound environment dynamics, the transition kernel must satisfy three fundamental properties. First, it must achieve \textit{measure preservation} to guarantee that the state space neither collapses nor explodes, ensuring all states remain uniformly accessible. Second, it must exhibit \textit{action mass preservation}, confirming that the agent's action signals are fully conserved during projection and no action is arbitrarily discarded. Finally, it must guarantee \textit{bounded variance transfer}, ensuring that the variance injected by the agent remains strictly bounded, preventing both vanishing and exploding signal. We formally define and prove these three characteristics below.

\subsection{Exact measure preservation}

The transition function relies on shifting the state by an action-conditional constant and applying a periodic triangle-wave activation. This operation is exactly measure-preserving on the unit hypercube.

\begin{theorem}[Exact preservation of uniform measure]
\label{thm:transition_exact}
Let the state space be the $N_s$-dimensional unit hypercube $\mathcal{X} = [0, 1]^{N_s}$. Assume the input state $S \sim \mathcal{U}(\mathcal{X})$. For any arbitrary fixed action $a$ and bias $b$, let $c(a, b) = b + aW \in \mathbb{R}^{N_s}$ be the resulting constant shift vector. Let $\mathcal{T}: \mathcal{X} \to \mathcal{X}$ be the transition function such that $S' = \mathcal{T}(S)$, defined component-wise by the triangle-wave activation $f(x) = \frac{1}{\pi} \arccos(\cos(2\pi x))$. Then, the output state maintains the exact uniform distribution, $S' \sim \mathcal{U}(\mathcal{X})$.
\end{theorem}

\begin{proof}
Let $c = c(a,b) \in \mathbb{R}^{N_s}$ be the fixed shift vector. We analyze the activation function component-wise. If the marginals are independent and uniformly distributed on $[0,1]$, the joint distribution is identically $\mathcal{U}(\mathcal{X})$.

Consider a single scalar component $s \sim \mathcal{U}(0,1)$ and its corresponding shift $c_i$. The operation first applies an affine shift $s + c_i$. The function $f(x) = \frac{1}{\pi} \arccos(\cos(2\pi x))$ is $1$-periodic and symmetric. Therefore, $f(s + c_i) = f((s + c_i) \bmod 1)$. 

Let $y = (s + c_i) \bmod 1$. Because translating a uniform random variable on a periodic domain (the torus $\mathbb{R}/\mathbb{Z}$) preserves the Haar measure, $y$ is exactly distributed as $\mathcal{U}(0,1)$.

We must now show that applying $f$ to $y$ preserves this uniform distribution. The function $f(y)$ piecewise evaluates to
\begin{equation}
f(y) = \begin{cases} 
2y & \text{if } y \in [0, 0.5] \\ 
2(1-y) & \text{if } y \in (0.5, 1] 
\end{cases}
\end{equation}

To find the probability density of $z = f(y)$, let $[u, v] \subseteq [0, 1]$ be an arbitrary interval. The pre-image $f^{-1}([u, v])$ consists of exactly two disjoint intervals in the domain of $y$: $[\frac{u}{2}, \frac{v}{2}]$ and $[1 - \frac{v}{2}, 1 - \frac{u}{2}]$. 

Because $y \sim \mathcal{U}(0,1)$, the probability mass of $z \in [u, v]$ is the sum of the Lebesgue measures of these pre-image intervals,
\begin{equation}
\mathbb{P}(z \in [u, v]) = \left(\frac{v}{2} - \frac{u}{2}\right) + \left(\left(1 - \frac{u}{2}\right) - \left(1 - \frac{v}{2}\right)\right) = v - u.
\end{equation}
Since the probability mass of any interval $[u, v]$ maps exactly to its length $v - u$, it follows that $z \sim \mathcal{U}(0,1)$. Because this mapping applies component-wise to independent variables, the joint distribution remains exactly $\mathcal{U}(\mathcal{X})$.
\end{proof}

\subsection{Action mass preservation}

\begin{proposition}[Action mass preservation]

Let $a_t \in [0, 1]^{N_a}$ denote the action vector at time step $t$, and let $p_t = a_t W \in \mathbb{R}^{N_s}$ be its linear projection into the pre-activation state space, where $W \in \mathbb{R}^{N_a \times N_s}$ is a weight matrix with non-negative entries ($W_{ij} \ge 0$) that is initialized to be row-stochastic ($\sum_{j=1}^{N_s} W_{ij} = 1$ for all $i$). The $L_1$ norm of the action signal is perfectly conserved during projection, such that $\|p_t\|_1 = \|a_t\|_1$.
\end{proposition}

\begin{proof}
By definition, the projected action vector is $p_t = a_t W$. Because the action is bounded in the unit hypercube ($a_{t,i} \ge 0$) and the weights are non-negative ($W_{ij} \ge 0$), all components of $p_t$ are strictly non-negative. 
Therefore, the total magnitude, or $L_1$ norm, of the projected signal is simply the sum of its components,
\begin{equation}
    \|p_t\|_1 = \sum_{j=1}^{N_s} p_{t,j} = \sum_{j=1}^{N_s} \left( \sum_{i=1}^{N_a} a_{t,i} W_{ij} \right)
\end{equation}
By exchanging the order of summation, we can isolate the row sums of the weight matrix $W$,
\begin{equation}
    \|p_t\|_1 = \sum_{i=1}^{N_a} a_{t,i} \left( \sum_{j=1}^{N_s} W_{ij} \right)
\end{equation}
By the definition of the transition kernel's row-stochastic initialization, the sum of the weights associated with each action dimension equals unity, $\sum_{j=1}^{N_s} W_{ij} = 1$. Substituting this into the equation yields
\begin{equation}
    \|p_t\|_1 = \sum_{i=1}^{N_a} a_{t,i} (1) = \sum_{i=1}^{N_a} a_{t,i} = \|a_t\|_1
\end{equation}
Thus, the total action mass is exactly conserved regardless of the dimensionalities $N_a$ and $N_s$.
\end{proof}

\subsection{Bounded variance transfer}

\begin{proposition}[Bounded variance transfer]
Assume the agent acts according to a maximum-entropy exploratory policy such that $a_t \sim \mathcal{U}(0,1)^{N_a}$. Let $\Sigma_p = \text{Cov}(p_t)$ denote the covariance matrix of the projected action signal $p_t = a_t W$. The total variance injected into the state space, defined by the trace of $\Sigma_p$, is strictly bounded by $\frac{N_a}{12 N_s} \le \text{Tr}(\Sigma_p) \le \frac{N_a}{12}$.
\end{proposition}

\begin{proof}
Under a maximum-entropy exploratory policy uniformly covering the action space, the components of $a_t$ are independent and identically distributed as $\mathcal{U}(0,1)$. The covariance matrix of the action vector is therefore a scaled identity matrix, $\Sigma_a = \text{Cov}(a_t) = \frac{1}{12} I_{N_a}$.

The covariance matrix of the linearly projected signal $p_t = a_t W$ is given by
\begin{equation}
    \Sigma_p = \text{Cov}(a_t W) = W^\top \Sigma_a W = W^\top \left( \frac{1}{12} I_{N_a} \right) W = \frac{1}{12} W^\top W
\end{equation}

The total variance injected into the environment dynamics is defined by the trace of $\Sigma_p$, 
\begin{equation}
    \text{Tr}(\Sigma_p) = \frac{1}{12} \text{Tr}(W^\top W) = \frac{1}{12} \|W\|_F^2
\end{equation}
where $\|W\|_F^2 = \sum_{i=1}^{N_a} \sum_{j=1}^{N_s} W_{ij}^2$ is the squared Frobenius norm of the weight matrix.

To bound $\text{Tr}(\Sigma_p)$, we must bound the sum of squares for each row $i$, subject to the row-stochastic constraints $W_{ij} \ge 0$ and $\sum_{j=1}^{N_s} W_{ij} = 1$. 

\textbf{Lower Bound:} By the Cauchy-Schwarz inequality, $\left(\sum_{j=1}^{N_s} W_{ij}\right)^2 \le N_s \sum_{j=1}^{N_s} W_{ij}^2$. Since the row sum is 1, this simplifies to $1 \le N_s \sum_{j=1}^{N_s} W_{ij}^2$, yielding a minimum of $\frac{1}{N_s}$ when the weights are perfectly uniform ($W_{ij} = \frac{1}{N_s}$).

\textbf{Upper Bound:} Because $W_{ij} \ge 0$ and $\sum_{j=1}^{N_s} W_{ij} = 1$, each individual weight is bounded by $W_{ij} \le 1$. Consequently, $W_{ij}^2 \le W_{ij}$. Summing over $j$ yields $\sum_{j=1}^{N_s} W_{ij}^2 \le \sum_{j=1}^{N_s} W_{ij} = 1$. The supremum of 1 is approached as the projection collapses to a single state dimension.

Summing these individual row bounds across all $N_a$ action dimensions provides the bounds for the squared Frobenius norm,
\begin{equation}
    \sum_{i=1}^{N_a} \left( \frac{1}{N_s} \right) \le \|W\|_F^2 \le \sum_{i=1}^{N_a} (1) \implies \frac{N_a}{N_s} \le \|W\|_F^2 \le N_a
\end{equation}

Substituting these bounds into the trace equation yields the final bounds on the transferred variance,
\begin{equation}
    \frac{N_a}{12 N_s} \le \text{Tr}(\Sigma_p) \le \frac{N_a}{12}
\end{equation}
This guarantees that the total exploratory variance injected into the transition kernel is strictly positive (bounded away from zero, preventing signal vanishing) and strictly bounded from above (preventing signal explosion), ensuring stable information transfer.
\end{proof}

\subsection{Choice of activation function}
The primary role of the activation function within the transition kernel is to project the unbounded $N_s$-dimensional pre-activation vectors from $\mathbb{R}^{N_s}$ onto the unit hypercube $[0, 1]^{N_s}$. This may theoretically also be achieved by various other functions than the selected triangle wave. Prominent examples of such functions are sigmoid or modulo. However, the sigmoid function saturates for high absolute input values. This possibly leads to a collapse of the state space, which corrupts the learning task, as optimal training performance may be achieved by learning the optimal actions for a fraction of the state space. The modulo operator ($x \pmod 1$) prevents this collapse, but implies a toroidal topology where the edges wrap around causing a jump discontinuity from $1 \to 0$. The triangle wave activation addresses both of these challenges. It effectively prevents the state space collapse induced by activation function saturation, and yields a smooth topology: Trajectories hitting the boundary are reflected rather than teleported. 

\subsection{Lipschitz continuity}
Beyond merely avoiding jump discontinuities, the reflective property of the triangle wave activation ensures that the environment dynamics remain mathematically stable and bounded. Specifically, it guarantees that the transition kernel is globally Lipschitz continuous. Formally, the transition kernel $\mathcal{T}(s,a)$ is a composition of an affine transformation $g(s,a)=s+aW+b$ and the triangle wave activation $\psi(x)=\frac{1}{\pi}\arccos(\cos(2\pi x))$. The Lipschitz constant of a composition of functions is bounded by the product of their respective Lipschitz constants. The piecewise linear function $\psi(x)$ has a bounded derivative of $\pm 2$, yielding a Lipschitz constant $L_\psi=2$. The affine transformation $g(s,a)$ is Lipschitz continuous with constant $L_g \leq \max(1, \|W\|_2)$, corresponding to the identity mapping for states and the induced norm of $W$ for actions. Therefore, the joint Lipschitz constant of the transition kernel is strictly bounded by $L_T \leq 2\max(1, \|W\|_2)$, providing a rigorous guarantee of smooth, predictable state transitions.

\subsection{Suitability as an optimal policy}
\label{sec:learnability}

While both the transition function and the optimal policy achieve approximate measure preservation, their practical implementations present differing optimization landscapes. 

While the exact measure preservation of the transition function is highly appealing, the function does not serve as a suitable approach for designing a learning target. The learning target must allow for a configuration of mapping complexity. Scaling the mathematical complexity of the transition function inherently requires increasing the frequency of the triangle-wave oscillations (i.e., inducing more \emph{folds} in the state space). Due to the spectral bias of deep neural networks, whereby networks preferentially learn low-frequency functions and struggle to approximate highly oscillatory targets, this formulation becomes an extremely difficult learning target as complexity scales, and yields a challenging learning target even in its default configuration \citep{rahaman2019_spectral}. Furthermore, relying exclusively on oscillatory frequency for complexity modulation introduces a structural vulnerability: agents augmented with Fourier features--which explicitly mitigate spectral bias--could trivially neutralize the environment's complexity scaling.

The \gls{acr:dun} circumvents this limitation. Instead of relying on high-frequency geometric folding, it modulates complexity smoothly through architectural depth and orthogonal linear mixings. Although it trades exact measure preservation for asymptotic approximation, it produces a significantly smoother, more well-behaved optimization landscape.

\section{Extension to offline reinforcement learning}\label{sec:offline_rl}

To rigorously evaluate the performance of offline \gls{acr:rl} algorithms under varying degrees of dataset quality and environmental complexity, we design a controlled experimental pipeline. Our framework allows us to systematically degrade the quality of the behavior policy and analyze how offline \gls{acr:rl} algorithms recover or improve upon the data-generating policy.

\subsection{Experimental setup}

\paragraph{Dataset creation} We conduct our experiments using a \gls{acr:sme} in the default configuration. To generate offline datasets of varying quality, we employ a behavior policy that stochastically interpolates between the environment's ground-truth optimal policy and a structured noise policy, instantiated as a randomly initialized \gls{acr:dun}. At each time step $t$, we compute the optimal action $a^*_t$ using the optimal policy. Simultaneously, we compute a state-dependent noise action $\tilde{a}_t$ using the noise policy. The behavior policy samples an interpolation weight $\alpha_t \sim \mathcal{U}(0, \nu)$, where $\nu$ represents the maximum noise level. The final executed action $a_t$ is computed as
\begin{equation}
    a_t = (1 - \alpha_t) a^*_t + \alpha_t \tilde{a}_t.
\end{equation}
By dynamically sampling $\alpha_t$ at each step, the policy injects variable, state-conditioned noise. We collect $50{,}000$ transitions per dataset. To comprehensively map algorithm performance against data quality, we generate independent datasets across a grid of target policy complexities $\mathcal{C}_{\pi^{\star}} \in \{1, 10, 50\}$ and maximum noise levels $\nu \in \{0.0, 0.1, 0.5, 1.0\}$.

\paragraph{Training} We train two standard offline algorithms using the \textsc{d3rlpy} library \citep{seno2022d3rlpyofflinedeepreinforcement}, \gls{acr:bc} and \gls{acr:iql}. \gls{acr:bc} serves as our baseline, explicitly attempting to mimic the blended behavior policy without leveraging reward signals. \gls{acr:iql} is chosen for its ability to learn from sub-optimal data without querying \gls{acr:ood} actions. We use an expectile parameter of $\tau = 0.7$ to enable the algorithm to prioritize the higher-quality transitions within the blended dataset. Both algorithms are trained for $20{,}000$ gradient steps per dataset configuration, utilizing a batch size of $256$. 

\paragraph{Evaluation Protocol} To assess final performance, we measure $\tilde{r}_t$ in a live environment. For both \gls{acr:iql} and \gls{acr:bc}, across all noise and complexity configurations, we roll out 10 independent evaluation trajectories. We record the accumulated $\tilde{r}_t$ for each episode and compute its average across trajectories to obtain a performance metric. To quantify the algorithms' ability to stitch together optimal trajectories from noisy data, we compare the offline performance directly against the empirical average performance of the specific blended behavior policy used to collect the dataset.

\paragraph{Results}

As presented in Table \ref{tab:performance_results} and Figure~\ref{fig:offline}, the experiment reveals distinct behavioral regimes across the evaluated offline algorithms: For datasets with low-to-moderate noise and low-to-moderate complexity, both \gls{acr:bc} and \gls{acr:iql} perform on par with the data-generating policy. The core advantage of \gls{acr:iql} becomes evident in high-noise environments ($\nu \geq 0.5$). While \gls{acr:bc} continues to strictly imitate the degraded dataset, \gls{acr:iql} successfully filters out sub-optimal actions to surpass the behavior policy. For instance, at the maximum noise level ($\nu = 1.0$) and base complexity ($\mathcal{C}_{\pi^{\star}} = 1$), \gls{acr:iql} achieves an average return of $0.912$, significantly outperforming the dataset mean of $0.836$. This confirms \gls{acr:iql}'s ability to stitch together optimal trajectory fragments from highly noisy data. At the highest optimal policy complexity ($\mathcal{C}_{\pi^{\star}} = 50$), both algorithms suffer a severe performance collapse, regardless of the noise level. In these configurations, \gls{acr:bc} and \gls{acr:iql} fail to even match the dataset mean. This indicates that at extreme topological complexities, the models struggle to approximate the underlying target policy, likely due to the limited dataset size of $50{,}000$ transitions.

\begin{table}[H]
    \caption{Performance comparison in $\tilde{r}_t$ across noise levels and policy complexities (BP = Behavior policy, BC = Behavior cloning, IQL = Implicit Q-Learning.}
    \begin{center}
        \small
        \begin{tabular}{lcccc}
            \multicolumn{1}{l}{\bf $\nu$} & \multicolumn{1}{c}{\bf $\mathcal{C}_{\pi^{\star}}$} & \multicolumn{1}{c}{\bf BP} & \multicolumn{1}{c}{\bf BC} & \multicolumn{1}{c}{\bf IQL}
            \\ \hline \\
            0.0 & 1  & $1.000 \pm 0.00$ & $0.998 \pm 0.00$ & $0.997 \pm 0.00$ \\
            0.0 & 10 & $1.000 \pm 0.00$ & $0.979 \pm 0.01$ & $0.979 \pm 0.00$ \\
            0.0 & 50 & $1.000 \pm 0.00$ & $0.803 \pm 0.07$ & $0.798 \pm 0.07$ \\[1ex]
            
            0.1 & 1  & $0.982 \pm 0.01$ & $0.981 \pm 0.01$ & $0.981 \pm 0.01$ \\
            0.1 & 10 & $0.979 \pm 0.01$ & $0.971 \pm 0.01$ & $0.970 \pm 0.01$ \\
            0.1 & 50 & $0.983 \pm 0.01$ & $0.801 \pm 0.07$ & $0.799 \pm 0.07$ \\[1ex]
            
            0.25 & 1  & $0.957 \pm 0.03$ & $0.956 \pm 0.02$ & $0.958 \pm 0.01$ \\
            0.25 & 10 & $0.960 \pm 0.03$ & $0.955 \pm 0.02$ & $0.956 \pm 0.01$ \\
            0.25 & 50 & $0.958 \pm 0.03$ & $0.797 \pm 0.07$ & $0.792 \pm 0.07$ \\[1ex]
            
            0.5 & 1  & $0.901 \pm 0.07$ & $0.902 \pm 0.03$ & $0.938 \pm 0.02$ \\
            0.5 & 10 & $0.916 \pm 0.06$ & $0.913 \pm 0.03$ & $0.937 \pm 0.02$ \\
            0.5 & 50 & $0.918 \pm 0.06$ & $0.789 \pm 0.07$ & $0.791 \pm 0.07$ \\[1ex]
            
            1.0 & 1  & $0.836 \pm 0.11$ & $0.840 \pm 0.05$ & $0.912 \pm 0.03$ \\
            1.0 & 10 & $0.817 \pm 0.13$ & $0.818 \pm 0.06$ & $0.900 \pm 0.03$ \\
            1.0 & 50 & $0.840 \pm 0.11$ & $0.770 \pm 0.08$ & $0.780 \pm 0.07$ \\
        \end{tabular}
    \end{center}
    \label{tab:performance_results}
\end{table}

\begin{figure}[H]
    \centering
    \includegraphics[width=1\linewidth]{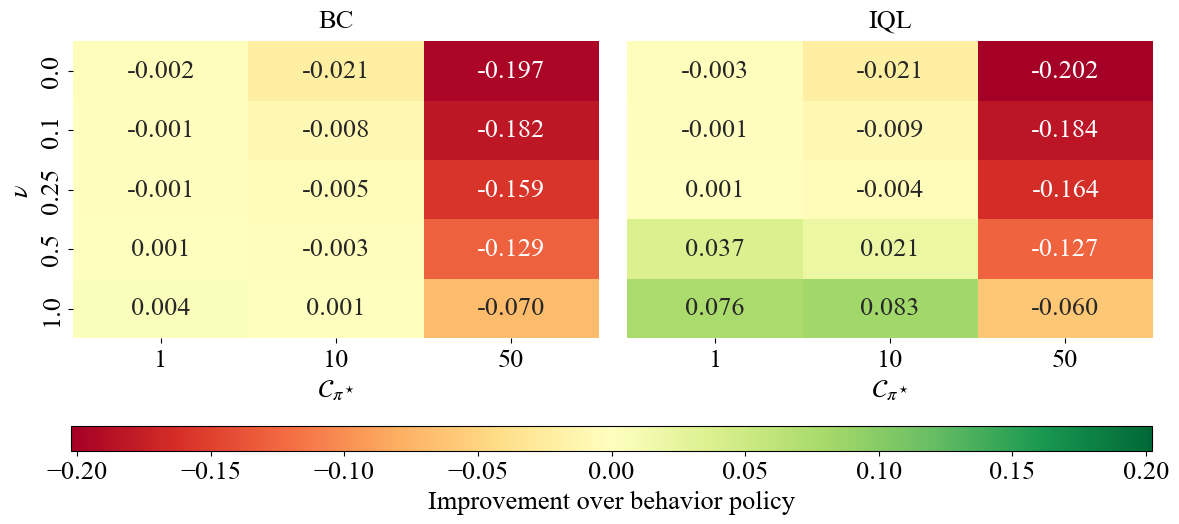}
    \caption{Improvement over the behavioral policy for Behavior Cloning (BC) and Implicit Q-Learning (IQL) by noise level $\nu$ and optimal policy complexity $\mathcal{C}_{\pi^{\star}}$.}
    \label{fig:offline}
\end{figure}

\newpage
\section{Ablations}\label{sec:ablations_fullpage}

\begin{figure}[h!]
    \centering
    \includegraphics[width=.9\linewidth]{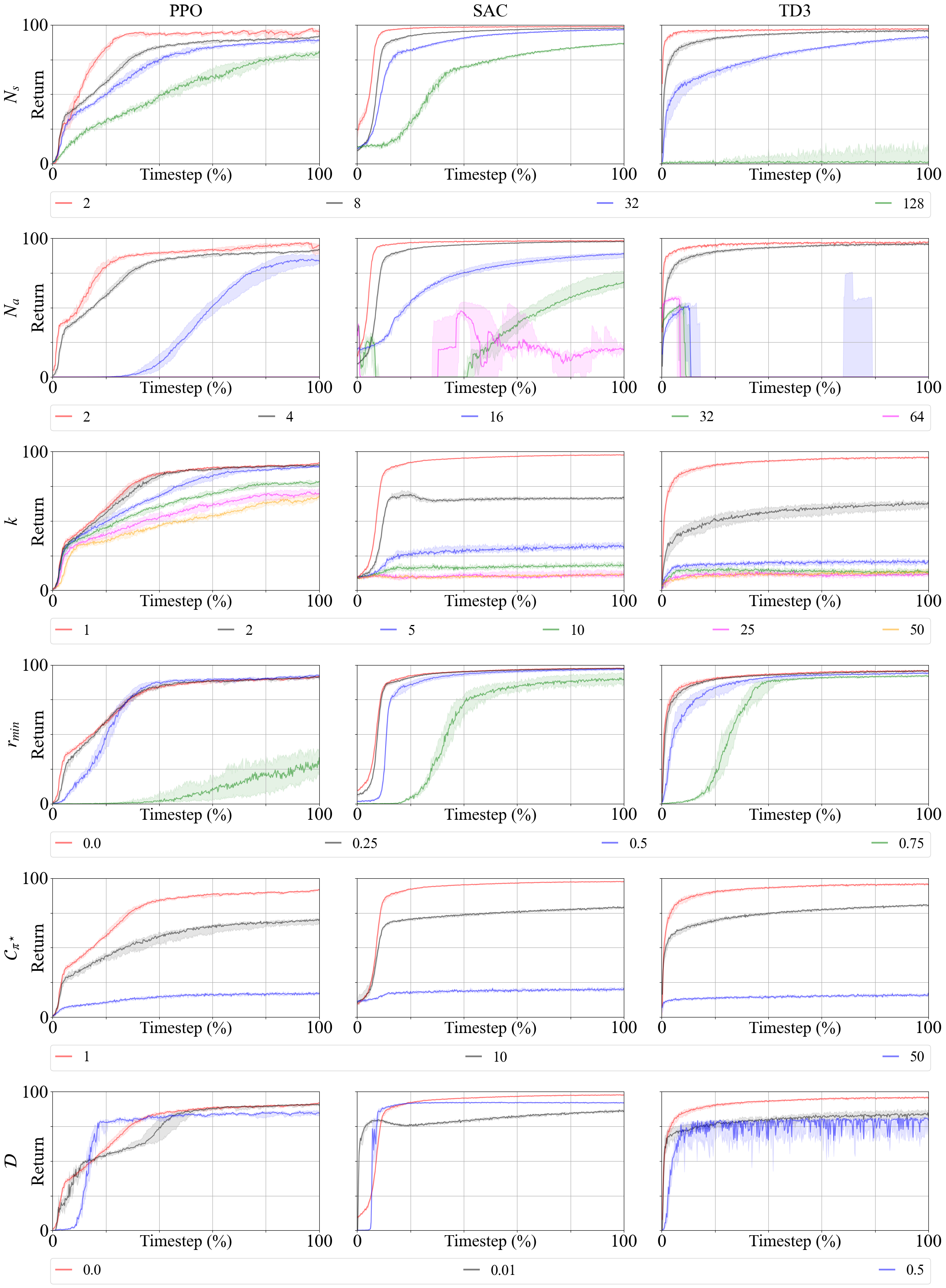}
    \caption{Ablations for PPO, SAC and TD3 across different task complexities over $10$ seeds. Curves indicate median evaluation performance, smoothed over $10$ points. Shaded areas indicate interquartile ranges.}
    \label{fig:ablations_fullpage}
\end{figure}

\end{document}